\documentclass[11pt]{article}
\usepackage{lipsum}
\usepackage{mathrsfs}
\usepackage{graphicx}
\usepackage{booktabs}
\usepackage{multirow}
\usepackage{bm}
\usepackage{hyperref}
\usepackage{float}
\usepackage{placeins}
\usepackage{subcaption}
\usepackage{enumitem}
\usepackage{mathtools}
\usepackage{wrapfig}
\usepackage{xurl}
\usepackage{adjustbox}
\usepackage{algorithm}
\usepackage{algpseudocode}
\usepackage{etoolbox}
\usepackage{amssymb}
\usepackage{amsthm}
\newtheorem{theorem}{Theorem}
\newtheorem{lemma}{Lemma}
\newtheorem{proposition}{Proposition}
\usepackage[normalem]{ulem}

\begin{document}

\author{Min Young Baeg, Yoon-Yeong Kim}

\title{PDE-regularized Dynamics-informed Diffusion with Uncertainty-aware Filtering for Long-Horizon Dynamics}
\date{}
\maketitle

\begin{abstract}
Long-horizon spatiotemporal prediction remains a challenging problem due to cumulative errors, noise amplification, and the lack of physical consistency in existing models. While diffusion models provide a probabilistic framework for modeling uncertainty, conventional approaches often rely on mean squared error objectives and fail to capture the underlying dynamics governed by physical laws. In this work, we propose PDYffusion, a dynamics-informed diffusion framework that integrates PDE-based regularization and uncertainty-aware forecasting for stable long-term prediction. The proposed method consists of two key components: a PDE-regularized interpolator and a UKF-based forecaster. The interpolator incorporates a differential operator to enforce physically consistent intermediate states, while the forecaster leverages the Unscented Kalman Filter to explicitly model uncertainty and mitigate error accumulation during iterative prediction. We provide theoretical analyses showing that the proposed interpolator satisfies PDE-constrained smoothness properties, and that the forecaster converges under the proposed loss formulation. Extensive experiments on multiple dynamical datasets demonstrate that PDYffusion achieves superior performance in terms of CRPS and MSE, while maintaining stable uncertainty behavior measured by SSR. We further analyze the inherent trade-off between prediction accuracy and uncertainty, showing that our method provides a balanced and robust solution for long-horizon forecasting.
\end{abstract}

\section{Introduction}\label{sec:intro}

In many applied scientific domains, accurately tracking and forecasting dynamical systems over long time horizons remains a fundamental and challenging problem. Traditional numerical approaches grounded in explicit mathematical modeling often incur substantial computational cost and become impractical for large-scale or long-term simulations. As a result, data-driven approaches based on deep neural networks have been widely explored as alternatives for long-horizon dynamical prediction, including Recurrent Neural Networks (RNNs)~\cite{elman1990finding} and Long Short-Term Memory (LSTM) networks~\cite{hochreiter1997long}.
However, their effectiveness in long-horizon forecasting remains limited due to vanishing gradients and cumulative prediction errors, due to their autoregressive structures that sequentially generate future states by conditioning on previously predicted outputs.

To alleviate these limitations, subsequent research has explored probabilistic and generative modeling approaches that aim to capture the full predictive distribution rather than producing a single deterministic trajectory. Among these, Generative Adversarial Networks (GANs)~\cite{goodfellow2014generative} were introduced to enhance sample diversity and model uncertainty in dynamical prediction. Despite their ability to generate diverse samples, however, GAN-based approaches suffer from well-known challenges, including training instability, mode collapse, and limited reproducibility, which hinder their reliability in long-term dynamical forecasting.

Diffusion models~\cite{ho2020denoising,song2020denoising} provide a principled alternative to autoregressive and adversarial approaches by learning the full conditional distribution through iterative noise injection and denoising. Their probabilistic nature enables uncertainty estimation, while their connections to ordinary differential equations (ODEs), partial differential equations (PDEs), and stochastic differential equations (SDEs) make them well-suited for physics-informed modeling. Owing to these advantages, diffusion models have become a dominant paradigm in high-dimensional generative modeling, particularly in static image synthesis. Until recently, however, their application to time-dependent and continuously evolving dynamical systems remained limited.

Recent studies have begun to extend diffusion models to long-horizon spatiotemporal forecasting, with TimeDiff~\cite{shen2023non}, MCVD~\cite{voleti2022mcvd}, and DYffusion~\cite{ruhling2023DYffusion} serving as representative examples. These works demonstrate that diffusion-based frameworks can successfully move beyond static generation and be adapted to dynamic prediction tasks involving long time scales.
Despite this progress, several fundamental challenges remain unresolved when applying diffusion models to continuous-time dynamical systems. Existing approaches typically rely on Gaussian diffusion processes and optimize training objectives based primarily on mean squared error, which inadequately capture the intrinsic physical and dynamical structure governed by underlying PDEs. As a consequence, these models may exhibit degraded distributional fidelity, limited point-wise tracking accuracy, and increased sensitivity to noise in long-horizon predictions. Furthermore, Gaussian-based diffusion formulations are not immune to cumulative error effects, particularly when deployed in iterative forecasting settings.

Motivated by these limitations, we propose a novel framework named \textbf{PDYffusion}, a PDE-based dynamics-informed diffusion model for long-horizon spatiotemporal forecasting. The proposed framework decomposes the prediction process into two complementary stages: interpolation and forecasting.
Specifically, we introduce a \textbf{PDE-regularized interpolator} that constructs physically consistent intermediate states by incorporating differential operators into the interpolation process. On top of this, we design an \textbf{uncertainty-aware forecaster} based on the Unscented Kalman Filter (UKF), which propagates predictions while explicitly modeling uncertainty and mitigating cumulative error.
By coupling physics-aware interpolation with uncertainty-aware forecasting, the proposed framework enables stable and coherent long-horizon prediction.

In a nutshell, the main contributions of this work are summarized as follows.
\begin{itemize}
    \item We introduce a PDE-regularized interpolation mechanism that enforces physically consistent intermediate trajectories, improving distributional fidelity and temporal coherence beyond standard MSE-based objectives.
    \item We develop an uncertainty-aware forecasting scheme based on the Unscented Kalman Filter, which effectively reduces cumulative error and stabilizes long-term predictions.
    \item We provide theoretical analysis for both components, establishing their convergence properties and supporting the proposed design.
\end{itemize}


The remainder of this paper is organized as follows. In Section~\ref{sec:Related}, we explain the prerequisite background for our research. 
In Section~\ref{sec:sec3}, we develop the proposed PDYffusion model using the PDE-based sampling and UKF-based filtering method, along with its mathematical analysis provided in Section~\ref{sec:sec4}. 
In Section~\ref{sec:sec5}, we present quantitative and qualitative experimental results on long-horizon dynamical datasets, demonstrating the effectiveness of our model. At last, Section~\ref{sec:sec6} concludes the paper.

\section{Related Work}
\label{sec:Related}

\subsection{Time-series Prediction using Deep Generative Model}\label{sec:2.1}
Deep generative models, including Variational Autoencoders (VAEs)~\cite{Kingma2013AutoEncodingVB}, Generative Adversarial Networks (GANs)~\cite{goodfellow2014generative}, Normalizing Flows~\cite{rezende2015variational}, and diffusion-based models~\cite{ho2020denoising,song2020denoising}, aim to learn the underlying probability distribution of data in order to generate new samples or perform related tasks such as interpolation, restoration, and conditional generation.
Among these families, diffusion models have recently emerged as a dominant paradigm due to their training stability, strong mode coverage, and scalability. By learning a gradual noise injection process and its corresponding denoising dynamics, diffusion models avoid adversarial optimization and exhibit favorable quality–diversity trade-offs across a wide range of domains, including images, audio, video, and scientific data. In most existing formulations, diffusion models are constructed using Gaussian forward processes and are trained by minimizing mean-squared error (MSE) objectives derived from denoising score matching.

Several recent works have extended Gaussian-based diffusion models with the backward reconstruction principle of Denoising Diffusion Probabilistic Models (DDPM)~\cite{ho2020denoising} for spatiotemporal and dynamical settings. TimeDiff~\cite{shen2023non} and MCVD~\cite{voleti2022mcvd} adapt the standard diffusion framework to long-horizon prediction by modeling temporal evolution through conditional denoising and masked reconstruction strategies. These approaches demonstrate that diffusion models can be applied beyond static data generation and can capture complex temporal dependencies over extended time scales. However, despite their empirical successes, these methods largely inherit the core assumptions of Gaussian diffusion models, including MSE-based training objectives and Markovian dynamics.

As a consequence, existing diffusion-based dynamical models often struggle to capture the underlying physical structure of continuous-time systems governed by partial differential equations. In particular, Gaussian noise assumptions and purely data-driven objectives limit their ability to enforce physical consistency, which leads to degraded long-horizon stability, distributional drift, and sensitivity to noise accumulation. These limitations motivate the incorporation of physics-informed structure and uncertainty-aware correction mechanisms into diffusion-based frameworks for dynamical prediction.

\subsection{DYffusion}\label{sec:2.2}

DYffusion~\cite{ruhling2023DYffusion} is a diffusion model specifically designed for predicting long-term dynamics. Unlike standard DDPM, DYffusion reformulates diffusion as an alternating sequence of interpolation and prediction steps. In this framework, the forward process is replaced by an interpolator, $I_\phi$, that reconstructs intermediate states between two observed time points by training with the following mean squared error (MSE) objective,
\begin{align}
\min_\phi ||I_\phi (x_t,x_{t+h},i)-x_{t+i}||_2^2 ,   
\label{eq:2.1}
\end{align}
where $i$ denotes the interpolation index between two observed data points, $x_t$ and $x_{t+h}$.
Then, the reverse process is implemented by a predictor, $F_\theta$, which is trained to forecast future states conditioned on the interpolated trajectory with the following objective,
\begin{align}
    \min_\theta ||F_\theta(I_\phi,i)-x_{t+h}||_2^2,
\label{eq:2.2}    
\end{align}
while the interpolator remains frozen.

During inference, DYffusion adopts a cold sampling~\cite{bansal2023cold} strategy which alternates between forecasting and interpolation without injecting stochastic noise. In this framework, the forecaster first predicts the terminal state as $\hat{x}_{t+h}=F_{\theta}(I_\phi,i_n)=F_{\theta}(\hat{x}_{t+i_n},i_n)$\footnote{Here, the index $i_n$ denotes that the intermediate time step is independently sampled at each iteration, reflecting the stochastic selection of interpolation points during forecaster learning.}. Then, the interpolator refines intermediate states based on both the initial observation and the predicted future. Specifically, the intermediate state is updated as
\begin{align}
    \hat{x}_{t+i_{n+1}} = \hat{x}_{t+i_n}+I_\phi(x_t,\hat{x}_{t+h},i_{n+1})-I_\phi(x_t,\hat{x}_{t+h},i_n),
\end{align}
which can be interpreted as a consistency-based correction that progressively aligns the trajectory with the interpolated path.
Unlike conventional autoregressive diffusion models that rely on noisy transitions, cold sampling operates entirely in the data space, enabling deterministic refinement and improved interpretability.

Despite these advantages, DYffusion inherits several structural limitations that restrict its applicability.
Since both interpolation and prediction stages are trained using MSE-based objectives, which prioritize point-wise reconstruction accuracy, they fail to encode the underlying dynamical constraints typically described by partial differential equations.
Furthermore, although DYffusion mitigates some aspects of autoregressive error accumulation, its predictor still operates on intermediate states that are estimates produced by the model. Error introduced during interpolation can therefore propagate through successive prediction-interpolation cycles, leading to increased sensitivity to noise naturally lying in long-time dynamics.

\subsection{PDE-based Sampling}\label{sec:2.3}

PDE-based sampling~\cite{khristenko2019analysis} aims to efficiently generate Gaussian random fields with prescribed covariance structures, particularly in settings where direct covariance factorization~\cite{davis1987production} or Karhunen–Loève expansion~\cite{schwab2006karhunen} becomes computationally prohibitive due to non-local dependencies.
A commonly considered example is the zero-mean Gaussian random field $u$ with Mat\'ern covariance. 
Let $D \subset \mathbb{R}^d$ be an open, simply connected, bounded domain; and consider a probability space $(\Omega,\mathcal{A},\mathbb{P})$, with a set of events $\Omega$, a $\sigma$-algebra $\mathcal{A}$, and a probability measure $\mathbb{P}$.
The Mat\'ern covariance function is defined as:
\begin{align}
C(x,y) &= \sigma_c^{2}\, M_{\nu}\!\bigl(\kappa \,\|x-y\|_{2}\bigr)\ , 
\forall x,y \in D ,
\label{eq:2.3}
\end{align}
where $\sigma_c^2$ denotes the marginal variance, $\kappa=\tfrac{\sqrt{2\nu}}{\rho}$ is related to the correlation length $\rho$, and $M_{\nu}(\cdot)$ is the unit Mat\'ern function given by: 
\begin{align}
M_{\nu}(x) &= \frac{x^{\nu} K_{\nu}(x)}{2^{\,\nu-1}\,\Gamma(\nu)},
\label{eq:2.4}
\end{align}
with $K_{\nu}$ denoting the modified Bessel function of the second kind.


A key insight from Whittle~\cite{de43fa33-b182-36b5-bc8c-fe95fc0a9e2f,whittle1963stochastic} is that a Gaussian random field with Mat\'ern covariance can be equivalently characterized as the solution to a stochastic partial differential equation (SPDE). Specifically, the field $u$ satisfies 
\begin{align}
\bigl(E-\kappa^{-2}\Delta\bigr)^{-\alpha/2} u &= \eta\, W,
\label{eq:2.5}
\end{align}
where $\alpha=\nu+\tfrac{d}{2}$, $ \eta^{2} = \sigma_c^{2} (4\pi)^{d/2}\,
\frac{\Gamma\!\bigl(\nu + \tfrac{d}{2}\bigr)}{\kappa^{d}\,\Gamma(\nu)} $ is a normalization constant, and $W$ denotes Gaussian white noise.
This formulation admits an explicit solution representation as:
\begin{align}
u = \eta \bigl(E-\kappa^{-2}\Delta\bigr)^{\alpha/2} W.
\label{eq:2.5}
\end{align}
Thus, the essence of PDE-based sampling is to generate a structural random field from white noise using a PDE structure.

\subsection{Unscented Kalman Filter}\label{sec:2.4}
The Kalman Filter~\cite{kalman1960new} is a standard framework for sequential state estimation under Gaussian uncertainty. Since the standard formulation assumes linear dynamics, several nonlinear extensions have been proposed~\cite{jazwinski1966filtering, evensen1994sequential}. Among them, the Unscented Kalman Filter (UKF)~\cite{julier2004unscented} provides an effective compromise between accuracy and computational efficiency by propagating sigma points through nonlinear dynamics.

For the $n$-dimensional state $x_k\sim\mathcal{N}(x_k, P_k)$ and the $m$-dimensional observation $z_k$ at the $k^{th}$ time index, the nonlinear state-measurement model is given by the following. 
\begin{align}
x_k &= f(x_{k-1}) + w_k, \quad\text{where} \quad
w_k \sim\mathcal \mathcal{N}(0,Q), \label{eq:2.7}\\
z_k &= h(x_k) + v_k, \quad \text{where}\quad v_k \sim \mathcal \mathcal{N}(0,R).
\label{eq:2.8}
\end{align}
Here, $f$ denotes the dynamics function, $h$ denotes the measurement function, $Q$ denotes the system error covariance, and $R$ denotes the measurement error covariance.

Then, the set of sigma points, denoted as $\{X_k^{(i)}\}$, are sampled as:
\begin{align}
\mathcal X_{k}^{(0)} &= \hat x_{k}, \label{eq:2.10}\\
\mathcal X_{k}^{(i)} &= \hat x_{k} + \bigl[\sqrt{(n+\lambda)P_{k}}\bigr]_i, \quad i=1,\dots,n, \label{eq:2.11}\\
\mathcal X_{k}^{(i)} &= \hat x_{k} - \bigl[\sqrt{(n+\lambda)P_{k}}\bigr]_{i-n}, \quad i=n+1,\dots,2n, \label{eq:2.12}
\end{align}
where $\lambda= \alpha^2 (n+\kappa) - n$ is defined with predefined hyperparameters $\alpha$ and $\kappa$. Each sigma point is then assigned the corresponding weight as:
\begin{align}
W_m^{(0)} &= \frac{\lambda}{n+\lambda}, &
W_c^{(0)} &= \frac{\lambda}{n+\lambda} + (1-\alpha^2+\beta), \label{eq:2.13}\\
W_m^{(i)} &= \frac{1}{2(n+\lambda)}, &
W_c^{(i)} &= \frac{1}{2(n+\lambda)}, \quad i=1,\dots,2n \label{eq:2.14}.
\end{align}
where $\beta$ is also a predefined hyperparameter.

In the prediction step of UKF, each sigma point is propagated by the dynamics and aggregated using the weights as:
\begin{align}
\mathcal X_{k+1\mid k}^{(i)} &= f\!\bigl(\mathcal X_{k}^{(i)}\bigr), \quad i=0,\dots,2n, \label{eq:2.15}\\
\hat x_{k+1\mid k} &= \sum_{i=0}^{2n} W_m^{(i)}\, \mathcal X_{k+1\mid k}^{(i)}, \label{eq:2.16}\\
P_{k+1\mid k} &= \sum_{i=0}^{2n} W_c^{(i)}
\bigl(\mathcal X_{k+1\mid k}^{(i)}-\hat x_{k+1\mid k}\bigr)
\bigl(\mathcal X_{k+1\mid k}^{(i)}-\hat x_{k+1\mid k}\bigr)^\top + Q.\label{eq:2.17}
\end{align}
The sigma points are also propagated by the measurement function and aggregated as:
\begin{align}
\mathcal Z_{k+1}^{(i)} &= h\!\bigl(\mathcal X_{k+1\mid k}^{(i)}\bigr), \quad i=0,\dots,2n, \label{eq:2.18}\\
\hat z_{k+1} &= \sum_{i=0}^{2n} W_m^{(i)}\, \mathcal Z_{k+1}^{(i)}, \label{eq:2.19}\\
S_{k+1} &= \sum_{i=0}^{2n} W_c^{(i)}
\bigl(\mathcal Z_{k+1}^{(i)}-\hat z_{k+1}\bigr)\bigl(\mathcal Z_{k+1}^{(i)}-\hat z_{k+1}\bigr)^\top + R, \label{eq:2.20}\\
P_{xz} &= \sum_{i=0}^{2n} W_c^{(i)}
\bigl(\mathcal X_{k+1\mid k}^{(i)}-\hat x_{k+1\mid k}\bigr)
\bigl(\mathcal Z_{k+1}^{(i)}-\hat z_{k+1}\bigr)^\top . \label{eq:2.21}
\end{align}

Finally, in the update step of UKF, the predicted states are redefined as follows:
\begin{align}
G &= P_{xz}\, S_{k+1}^{-1}, \label{eq:2.22}\\
\hat x_{k+1} &= \hat x_{k+1\mid k} + G (z_{k+1} - \hat z_{k+1}), \label{eq:2.23}\\
P_{k+1} &= P_{k+1\mid k} - G S_{k+1} G^\top \label{eq:2.24}.
\end{align}
where $G$ denotes the Kalman Gain.

\section{Methodology}
\label{sec:sec3}
In this section, we present PDYffusion for long-horizon dynamics prediction. The method consists of a PDE-regularized interpolator and an uncertainty-aware forecaster based on UKF, as depicted in Figure~\ref{fig:fig1} below. For convenience, the notation used in this section is summarized in Table~\ref{tab:tab1} of Appendix~\ref{app:notation}.
\begin{figure}[h]
  \centering
  \includegraphics[width=1.0\textwidth]{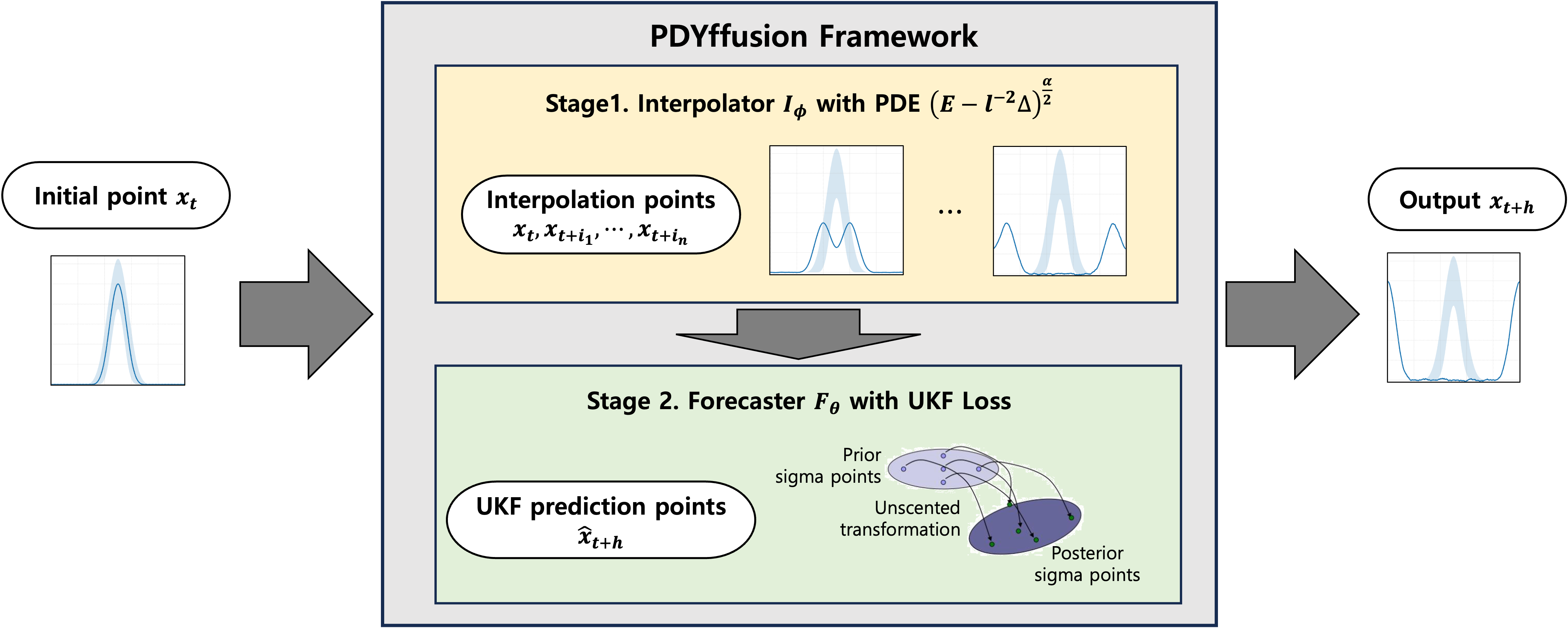}
  \caption{PDYffusion Framework. Given an intial state $x_t$, the interpolator generates intermediate states by enforcing physical structure through the operator $\bigl(E-l^{-2}\Delta\bigr)^{\alpha/2}$. These intermediate states are then used by the forecaster to predict the future state $x_{t+h}$, while a UKF-based prediction process models and corrects uncertainty via sigma points.}
  \label{fig:fig1}
\end{figure}

\subsection{PDE-regularized Interpolator}\label{sec:3.2}

Interpolation plays a central role in diffusion-based dynamical forecasting by refining intermediate states using information from future observations. 
In the prior diffusion-based framework for long-horizon prediction, or DYffusion~\cite{ruhling2023DYffusion}, this interpolation step serves as an analogous function to the forward process of the standard diffusion models.
Unlike conventional diffusion models that inject stochastic noise to drive data toward a standard Gaussian distribution, the difference in the interpolation of DYffusion is that it deterministically shifts the state $x_t$ toward the intermediate state $x_{t+i}$ via cold sampling. That is, the iterative inference scheme and its intermediate-state update is interpolator-driven as stated in Section~\ref{sec:2.2}.
Consequently, it progressively loses the information of $x_t$, performing a role similar to the forward process in traditional diffusion models.
However, when the interpolation model is trained solely by minimizing a mean squared error (MSE) objective as in Eq.~\eqref{eq:2.1}, it fails to capture the intrinsic structure of the underlying dynamics. Since the data considered in this work exhibit physically governed motion, purely point-wise error minimization is insufficient. This motivates the development of an interpolation mechanism that explicitly accounts for physical consistency. Consequently, we adopt the formulation in Eq.~\eqref{eq:2.5}. However, it should be noted that while Eq.~\eqref{eq:2.5} describes a probabilistic mechanism that generates samples directly from white noise, our approach does not rely on stochastic sampling. Instead, the intermediate state in our approach is constructed from observed data through a learned interpolation mapping, rather than being generated by white noise.

To this end, we formulate the interpolation problem as a PDE-constrained variational problem. Specifically, we seek a function $u(x,t)$ that interpolates between boundary states $x_t$ and $x_{t+h}$ while minimizing the functional:
\begin{align}
&J(u)=\frac{1}{2}\int_{t}^{t+h}\int_{\Omega}
\left|
\left(E-l^{-2}\Delta\right)^{\alpha/4}u(x,t)
\right|^{2}\,dx\,dt,\label{eq:3.1}
\\
&\text{where}\quad u(x,t) = x_t,
u(x,t+h) = x_{t+h}.
\end{align}
Here, $\Delta$ denotes the Laplacian operator on the spatial domain $\Omega$, while $l$ and $\alpha$ control the correlation length and smoothness of the interpolated trajectory, respectively. 

To derive the optimality condition, we consider a perturbation $u_{\varepsilon}(x,t) = u(x,t) + \varepsilon v(x,t)$ where the variation $v(x,t)$ satisfies homogeneous boundary conditions (i.e., $v(x,t)= 0$ and $v(x,t+h)=0$). By defining $M=\left(E-l^{2}\Delta\right)^{\alpha/4}$ for simplicity, we rewrite Eq.~\eqref{eq:3.1} as:
\begin{equation}
J(u+\varepsilon v)
=
\frac{1}{2}\int_{t}^{t+h}\int_{\Omega}
\left\|M\big(u(x,t)+\varepsilon v(x,t)\big)\right\|^{2}\,dx\,dt.\label{eq:3.2}
\end{equation}
Taking the first derivative of Eq.~\eqref{eq:3.2} by $\epsilon$ yields
\begin{align}
\frac{d}{d\varepsilon}J(u+\varepsilon v)
&=
\frac{1}{2}\int_{t}^{t+h}
\frac{d}{d\varepsilon}
\left[
\int_{\Omega}
\left\|M\big(u(x,t)+\varepsilon v(x,t)\big)\right\|^{2}\,dx
\right]dt,
\end{align}
and evaluating at $\epsilon=0$ gives
\begin{align}
\left.
\frac{d}{d\varepsilon}
\int_{\Omega}
\left\|M\big(u(x,t)+\varepsilon v(x,t)\big)\right\|^{2}\,dx
\right|_{\varepsilon=0}=
2\int_{\Omega}\left\langle Mu(x,t),\,Mv(x,t) \right \rangle\,dx.
\end{align}
Consequently, the first variation of $J$ can be written as
\begin{align}
\delta J(u,v)
&=
\int_{t}^{t+h}\int_{\Omega}
\left\langle Mu(x,t),\,Mv(x,t)\right\rangle\,dx\,dt
\\
&=
\int_{t}^{t+h}
\left\langle Mu(x,t),\,Mv(x,t)\right\rangle_{L^{2}(\Omega)}\,dt.
\end{align}

Under standard boundary conditions of PDE-based sampling such as Dirichlet or Neumann, the Laplace operator $\Delta$ is self-adjoint~\cite{arendt2007spectral}, which implies
\begin{align}
\langle Lu,Lv\rangle
=
\langle u,L^{2}v\rangle
=
\langle L^{2}u,v\rangle.
\end{align}
Using this property, the stationary condition $\delta J(u,v)=0$ for all $v$ yields
\begin{align}
\int_{t}^{t+h}
\left\langle M^{2}u(x,t),\,v(x,t)\right\rangle_{L^{2}(\Omega)}\,dt
&= 0,
\end{align}
which leads to the Euler-Lagrange equation as
\begin{equation}
M^{2}u(x,t)=
\left(E-l^{-2}\Delta\right)^{\alpha/2}u(x,t)=0.
\end{equation}
Since $\Delta$ is generally an elliptic operator, elliptic regularity ensures the existence and smoothness of the solution inside the domain. 

Assuming that the interpolator network $I_{\phi}$ approximates the solution $u(x,t)$, we impose the corresponding PDE constraint as
\begin{align}
\left(E-l^{-2}\Delta\right)^{\alpha/2}I_{\phi} &= 0.
\end{align}
Accordingly, the loss function of the newly proposed interpolator is defined as
\begin{align}
\min_{\phi}\;
\mathbb{E}_{\,i\sim U[1,h-1],\;x_{(t,t+i,t+h)}\sim X}
\Big[
\left|I_{\phi}(x_t,x_{t+h},i)-x_{t+i}\right|^{2}
+
\lambda
\left|\left(E-l^{-2}\Delta\right)^{\alpha/2}I_{\phi}\right|^{2}
\Big].\label{eq:final_loss}
\end{align}
In Eq.~\eqref{eq:final_loss}, the first term enforces reconstruction consistency by matching the interpolated state $I_{\phi}(x_t,x_{t+h},i)$ to the groundtruth intermediate state $x_{t+i}$, following the original DYffusion.
The novelty of the proposed PDYffusion comes from the second term, which introduces a PDE-based regularization that constrains the iterpolated trajectory to satisfy a physically meaningful differential operator.
Notably, this formulation departs from purely data-driven interpolation by explicitly incorporating a differential operator into the loss function. As a result, the interpolator is guided not only by point-wise supervision but also by the underlying physical structure of the dynamics.

\subsection{UKF-based Forecaster}\label{sec:3.3}

In DYffusion, the forecaster $F_\theta$ predicts the terminal state $x_{t+h}$ from the current state at each iteration, which is an $h$-step-ahead forecast.
Importantly, this current state is typically not the true observation $x_t$, but an intermediate estimate $\hat{x}_{t+i_{n}}$ produced by the iterative sampling loop. 
As a result, the forecaster operates on inputs that already contain approximation error. As this process is repeated over multiple forecasting-interpolation cycles, prediction errors can accumulate and propagate through the closed loop, leading to divergence or physically implausible patterns.
This risk is further amplified when using dropout-based ensembles. That is, while dropout improves sample diversity, individual ensemble members may drift into non-physical regions of the state space, thereby distorting the tails of the predictive distribution.

This observation motivated us to view the diffusion-based forecasting loop not only as a generative sampling process, but also as a sequential state-estimation problem in which prediction and correction should be explicitly disentangled. From this perspective, it becomes natural to consider Bayesian filtering as a corrective mechanism. Rather than relying solely on stochastic sampling or ensemble diversity to implicitly represent uncertainty, we propose to explicitly track and regulate predictive uncertainty through a filtering framework. This insight leads us to integrate a Kalman filtering strategy into the diffusion-based forecasting process.
Kalman Filter provides a principled framework for mitigating such closed-loop error amplification by explicitly separating prediction and update steps while tracking uncertainty via posterior mean and covariance.
However, since our predictor $F_\theta$ is nonlinear, a nonlinear extension must be considered rather than a standard Kalman filter. Among them, the Extended Kalman Filter (EKF)~\cite{jazwinski1966filtering} still relies on linear approximations and is computationally intensive due to second-order derivatives. The Ensemble Kalman Filter (EnKF)~\cite{evensen1994sequential} also avoids explicit linearization by propagating an ensemble of samples, but its computational cost scales with ensemble size.

At this end, to balance accuracy and efficiency, we adopt the Unscented Kalman filter (UKF)~\cite{julier2004unscented}, which approximates uncertainty propagation using a small set of deterministically chosen sigma points.
To start with, we first assume that the forecasted terminal state follows a Gaussian distribution as:
\begin{align}
x_{t+h}\sim
\mathcal{N}\!\big(\hat{x}_{t+h},\,P_{t+h}\big),
\end{align}
where $\hat{x}_{t+h}$ and $P_{t+h}$ denote the UKF-predicted mean and covariance, respectively. Under this assumption, the conditional negative log-likelihood of the true terminal state is given by 
\begin{align}
-\log p\!\left(x_{t+h}\mid \hat{x}_{t+h},\,P_{t+h}\right)
&=
\frac{d}{2}\log(2\pi)
+\frac12\log|P|
+\frac12 (x-\mu)^\top P^{-1}(x-\mu),
\end{align}
where $d$ denotes the dimensionality of the state space.
Defining the prediction error as $e_{t+h}:=x_{t+h}-\hat{x}_{t+h}$, we obtain the UKF-based loss term as:
\begin{align}
L_{\mathrm{UKF}}(\theta,\phi)
:= -\log p\!\left(x_{t+h}\mid \hat{x}_{t+h},\,P_{t+h}\right) 
=
\frac{d}{2}\log(2\pi)
+\frac12\log|P_{t+h}|
+\frac12 e_{t+h}^\top P_{t+h}^{-1} e_{t+h}.\label{eq:ukf_likelihood}
\end{align}
This likelihood-based formulation penalizes not only large prediction errors but also miscalibrated uncertainty estimates, encouraging consistency between the predicted mean trajectory and its associated covariance. In particular, the Mahalanobis term of Eq.~\eqref{eq:ukf_likelihood} adaptively weights errors according to predictive uncertainty, while the log-determinant term regularizes the spread of the covariance.

Finally, we jointly optimize the forecaster and interpolator by combining the standard prediction loss with the UKF-based uncertainty-aware correction:
\begin{align}
\min_{\theta,\phi}\Big(
\big\|F_{\theta}\!\big(I_{\phi},\,i_n\big)-x_{t+h}\big\|_2^{\,2}
+L_{\mathrm{UKF}}(\theta,\phi)\Big)
\label{eq:3.17}
\end{align}

We provide the pseudo code of our algorithm in Algorithm~\ref{alg:alg1}.

\begin{algorithm}[!t]
\caption{Training Procedure of PDYffusion}
\label{alg:PDYffusion}
\begin{algorithmic}[1]
\State \textbf{Input} networks $F_\theta, I_\phi$, norm $\|\cdot\|$, horizon $h$, schedule $\{i_n\}_{n=0}^{N-1}$
\State \textbf{Train interpolator network, $I_\phi$}
\State \quad Sample $i \sim \mathrm{Uniform}(\{1,\ldots,h-1\})$
\State \quad Sample $x_t, x_{t+i}, x_{t+h} \sim X$
\State \quad Optimize $\min_{\phi}\ \Bigl(\bigl\|I_{\phi}(x_t,x_{t+h},i)-x_{t+i}\bigr\|^{2}
+\lambda\,\bigl\|\bigl(E-l^{-2}\Delta\bigr)^{\alpha/2}\,I_{\phi}(x_t,x_{t+h},i)\bigr\|^{2}\Bigr)$
\State \textbf{Train forecaster network, $F_\theta$}
\State \quad Freeze Interpolator network $I_\phi$ 
\State \quad Sample $n \sim \mathrm{Uniform}(\{0,\ldots,N-1\})$ and $x_t, x_{t+h} \sim X$
\State \quad Optimize $\min_{\theta,\phi}\Big(
\big\|F_{\theta}\!\big(I_{\phi},\,i_n\big)-x_{t+h}\big\|_2^{\,2}
+L_{\mathrm{UKF}}(\theta,\phi)\Big)$
\end{algorithmic}
\label{alg:alg1}
\end{algorithm}

\section{Mathematical Analysis}\label{sec:sec4}

To better understand the theoretical properties of the proposed PDYffusion framework, we analyze the behavior of its two core components: the PDE-based interpolator and the filtering-based forecaster. The interpolator is designed to incorporate the underlying physical structure of the dynamics through PDE regularization, while the forecaster stabilizes long-horizon prediction using an uncertainty-aware filtering mechanism.

In this section, we first analyze the interpolator and show that incorporating the PDE term improves the distributional approximation compared to a purely MSE-based interpolator. We then analyze the forecaster equipped with the UKF-based objective and show that the proposed loss formulation theoretically converges to zero under mild assumptions.

\subsection{Interpolator Analysis}\label{sec:3.4.1}

In Section~\ref{sec:3.2}, we introduced a PDE-regularized interpolator designed to capture the underlying dynamics of the data. Unlike conventional interpolation methods that rely solely on reconstruction loss, the proposed approach incorporates a differential operator derived from the governing PDE. This encourages the interpolated trajectory to follow the physical structure of the dynamical system.
In this section, we analyze the theoretical implications of the PDE regularization term. Specifically, we will show that the prediction error of PDYffusion is bounded by the error of DYffusion by using Maximum Mean Discrepancy (MMD)~\cite{gretton2012kernel} as the metric.

Since MMD is calculated in the Reproducible Kernel Hilbert Space (RKHS)~\cite{aronszajn1950theory}, we first show that the Mate\'rn kernel used in PDYffusion is positive semi-definite.
\begin{proposition}[Positive Semi-definite of Mate\'rn Kernel]\label{pro:1}
Mat\'ern Kernel is positive semi-definite.
\end{proposition}

\begin{proof}
Recall the Mat\'ern kernel as Eq.~\eqref{eq:2.3}, whose power spectrum density becomes 
\begin{align}
S(k) = \sigma_c^2\frac{(2\pi)^\frac{d}{2}\Gamma(\alpha)l^(2v)}{\Gamma(v)(l^2+||k||^2)^\alpha},
\end{align}
where $d$ is dimension, $k$ is the frequency vector, and $\alpha=v+\frac{d}{2}$~\cite{ruhling2023DYffusion,de43fa33-b182-36b5-bc8c-fe95fc0a9e2f,whittle1963stochastic} . 
According to Bochner's Theorem, if a function transformed into the complex domain is positive semi-definite, then the original function is also positive semi-definite.
Having said that, since $S(k)\geq0$ holds, we can conclude that Mat\'ern kernel is also positive semi-definite.
\end{proof}

Proposition~\ref{pro:1} shows that the Mat\'ern kernel, which we used as a similarity metric in the interpolator of the PDYffusion framework, is positive semi-definite. Having said that, Mat\'ern kernel belongs to RKHS, which allows us to calculate MMD as a distance metric.

\begin{lemma}[Error Bound of Interpolator in Hilbert Space]\label{lem:lem4.2}

Let $\mathcal{H}$ be Hilbert space, and let $A=\left(E-l^{-2}\Delta\right)^{\alpha/2}:\mathcal{H}\xrightarrow{}\mathcal{H}$ be a self adjoint and positive semi-definite operator.
Let the outputs of the baseline interpolator and the PDE-regularized interpolator be expressed as
\begin{align}
    I_{\mathrm{base}}(x_t,x_{t+1},i) = u(x_t,x_{t+1},i)+\epsilon_{base},\quad I_{\phi}(x_t,x_{t+1},i) =u(x_t,x_{t+1},i)+\epsilon_{\phi},\label{eq:I_definition}
\end{align}
where $u$ is the groundtruth solution and $\epsilon_{base}, \epsilon_{\phi}$ are zero-mean random errors with covariance matrices $\Sigma_{base}$ and $\sigma_{\phi}$, respectively.
Assume that the PDE regularization enforces the constraint $Au=0$. Then the error covariance satisfies
\begin{align}
    \Sigma_{\phi}\le\Sigma_{base}.
\end{align}   
\end{lemma}

\begin{theorem}[Reconstruction Error Bound of Interpolator]\label{thm:thm3.1}

Let $P_\phi$, $P_{base}$, and $P_u$ denote the distributions induced by the PDE-regularized interpolator of PDYffusion, the baseline interpolator by DYffusion, and the groundtruth solution, respectively.
Then the following inequality holds:
\begin{align}
\mathbb{E}\!\left[\mathrm{MMD}_k^2\!\left(P_{\phi},P_u\right)\right]
\le
\mathbb{E}\!\left[\mathrm{MMD}_k^2\!\left(P_{\mathrm{base}},P_u\right)\right].
\label{eq:3.18}
\end{align}
This implies that the PDE-regularized interpolator produces a distribution closer to the groundtruth dynamics in the RKHS induced by the kernel $k$.
\end{theorem}

\begin{proof}
Let $u(x_t,x_{t+1},i)$ be the groundtruth solution of the underlying PDE, and let $I_{base}$ and $I_{\phi}$ denote the baseline interpolator and the proposed PDE-regularized interpolator, respectively. We model both interpolators as perturbations of the groundtruth solution as:
\begin{align}
    I_{base}=u+\epsilon_{base}, \quad I_{\phi}=u+\epsilon_{\phi},
\end{align}
where $\epsilon_{base}\sim \mathcal{N}(0,\Sigma_{base})$ and $\epsilon_{\phi}\sim \mathcal{N}(0,\Sigma_{\phi})$ represent interpolation errors, following Eq.~\eqref{eq:I_definition}.
Also, let $P_{base}$, $P_{\phi}$, and $P_u$ denote the distributions induced by $I_{base}$, $I_{\phi}$, and u, respectively.

We analyze the discrepancy between these distributions using the squared Maximum Mean Discrepancy (MMD), which is defined as:
\begin{align}\label{eq:MMD_original}
\mathrm{MMD}^2_k(P,Q)
&:= \mathbb{E}_{X,X'\sim P}\big[k(X,X')\big]
 + \mathbb{E}_{Y,Y'\sim Q}\big[k(Y,Y')\big]
 - 2\,\mathbb{E}_{X\sim P,\,Y\sim Q}\big[k(X,Y)\big],
\end{align}
where $k(\cdot,\cdot)$ denotes the kernel function, and $P$ and $Q$ denote arbitrary distributions to calculate the discrepancy.
By prior work~\cite{sriperumbudur2010hilbert}, which describes the MMD equation with a spectral density function $S_k(\cdot)$, Eq.~\eqref{eq:MMD_original} is reformulated as
\begin{align}
\qquad 
\mathrm{MMD}^2_k(P,Q)
&= \int_{\mathbb{R}^d} S_k(w)\,\big|\varphi_P(w)-\varphi_Q(w)\big|^2\,dw,
\label{eq:eq3.24}
\end{align}
where $\varphi_P(w):= \mathbb{E}_{X\sim P}\!\left[e^{i w^\top X}\right]$ denotes characteristic function~\cite{levy1925calcul} of the distribution $P$.

For the baseline interpolator, we have
\begin{align}
 \varphi_{P_{base}}(w) = E[e^{iw^{\top}(u+\epsilon_{base})}] = e^{iw^\top u}\cdot E[e^{iw^\top\epsilon_{base}}].   
\end{align}
Since $\epsilon_{base}\sim \mathcal{N}(0,\Sigma_{base})$, we get
\begin{align}
 E[e^{iw^\top\epsilon_{base}}]=e^{-\frac{1}{2}w^\top\Sigma_{base}w}.
\end{align}
Thus, the characteristic function for the base interpolator becomes
\begin{align}
    \varphi_{P_{base}}(w)=\varphi_{P_{u}}(w)e^{-\frac{1}{2}w^{\top}\Sigma_{base}w}.\label{eq:char_base}
\end{align}
Similarly, the characteristic function for the PDE-regularized interpolator becomes
\begin{align}
    \varphi_{P_{\phi}}(w)=\varphi_{P_{u}}(w)e^{-\frac{1}{2}w^{\top}\Sigma_{\phi}w}.\label{eq:char_pde}
\end{align}
Subtracting $\varphi_{P_u}(w)$ from Eq.~\eqref{eq:char_base} and Eq.~\eqref{eq:char_pde} yields the followings, respectively:
\begin{align}
    \varphi_{P_{base}}(w)-\varphi_{P_{u}}(w) &= \varphi_{P_{u}}(w)\cdot(e^{-\frac{1}{2}w^{\top}\Sigma_{base}w}-1),
    \label{eq:4.6}\\
     \varphi_{P_{\phi}}(w)-\varphi_{P_{u}}(w) &= \varphi_{P_{u}}(w)\cdot(e^{-\frac{1}{2}w^{\top}\Sigma_{\phi}w}-1).
\end{align}
Taking the squared magnitudes, we get
\begin{align}
    |\varphi_{P_{base}}(w)-\varphi_{P_{u}}(w)|^{2} &= |\varphi_{P_{u}}(w)|^{2}\cdot|(e^{-\frac{1}{2}w^{\top}\Sigma_{base}w}-1)|^{2},\label{eq:base_diff}\\
    |\varphi_{P_{\phi}}(w)-\varphi_{P_{u}}(w)|^{2} &= |\varphi_{P_{u}}(w)|^{2}\cdot|(e^{-\frac{1}{2}w^{\top}\Sigma_{\phi}w}-1)|^{2}.\label{eq:pde_diff}
\end{align}

Now, define the function, $g(a)=(1-e^{-\frac{a}{2}}), a\geq0$, which is monotonically increasing in $a$. Therefore, if $w^{\top}\Sigma_{\phi}w\le w^{\top}\Sigma_{base}w$, then it holds that
\begin{align}
|(e^{-\frac{1}{2}w^{\top}\Sigma_{\phi}w}-1)|^{2} \le |(e^{-\frac{1}{2}w^{\top}\Sigma_{base}w}-1)|^{2}.\label{eq:monoton}
\end{align}
Using the covariance ordering of $\Sigma_\phi \leq \Sigma_{base}$ from Lemma~\ref{lem:4.2}, Eq.~\eqref{eq:monoton} holds for all $w\in \mathbb{R}^d$. Hence, by recalling Eq.~\eqref{eq:base_diff} and Eq.~\eqref{eq:pde_diff}, we get
\begin{align}
    |\varphi_{P_{\phi}}(w)-\varphi_{P_{u}}(w)|^{2} \le |\varphi_{P_{base}}(w)-\varphi_{P_{u}}(w)|^{2}.
\end{align}
Multiplying both sides by the non-negative spectral density, $S_k(w)$, and integrating over $\mathbb{R}^d$, we obtain
\begin{align}
    \int_{\mathbb{R}^d} S_k(w)\big|\varphi_{P_{\phi}}(w)-\varphi_{P_{u}}(w)\big|^2\,dw \le \int_{\mathbb{R}^d} S_k(w)\big|\varphi_{P_{base}}(w)-\varphi_{P_{u}}(w)\big|^2\,dw,
\end{align}
which implies
\begin{align}
    \mathrm{MMD}^2_k(P_{\phi},P_{u}) \le \mathrm{MMD}^2_k(P_{base},P_{u}).
\end{align}
Taking expectation over the data distribution completes the proof.
\end{proof}

Theorem~\ref{thm:thm3.1} shows that incorporating PDE-based regularization reduces the discrepancy between the predicted distribution and the groundtruth distribution in RKHS, compared to the base interpolator trained with MSE loss only. This provides a theoretical justification for the improved stability and physical consistency observed in PDYffusion.

\subsection{Forecaster Analysis}
\label{sec:3.4.2}
In Section~\ref{sec:3.2}, we introduced a forecaster model equipped with UKF. In this section, we analyze the theoretical properties of the proposed forecaster and show that the training objective leads to a well-defined convergence behavior.
To establish the main result, we first introduce two auxiliary lemmas.

\begin{lemma}[Log-determinant difference formula]\label{lem:lem3.2}
    Let $A$ and $B$ be positive definite matrices. Then, the following identity holds.
    \begin{align}
    \log \det(A)-\log \det(B) = \int_{0}^{1} \operatorname{tr}\!\left[(B+s(A-B))^{-1}(A-B)\right]\, ds\nonumber
    \end{align}    
\end{lemma}

\begin{lemma}[Uniform spectral bounds] \label{lem:lem3.3}
    Let $P_{k}$ denote the covariance matrix generated by the UKF and let $Q$ denote the reference covariance matrix. Then, the eigenvalues of $P_{k}$ and $Q$ can be bounded $\frac{m}{2}\leq \lambda(P_k),\lambda(Q)\leq 2M$.
\end{lemma}

Lemma~\ref{lem:lem3.2} relates the difference between two log-determinants to a trace integral and is commonly used to analyze the sensitivity of covariance matrices, while Lemma~\ref{lem:lem3.3} guarantees numerical stability when evaluating matrix inverses and log-determinants. Appendix~\ref{proof:lem3.2} and~\ref{proof:lem3.3} provide the detailed proof for both Lemmas, respectively.

\begin{theorem}[Convergence of the UKF-based forecaster]\label{thm:thm3.4}

Let $F_{\theta}(I_{\phi},i_n)=\hat{x}_{t+h}$ denote the prediction of the proposed forecaster, and let $P_{t+h}$ be the corresponding covariance matrix obtained from the Unscented Kalman Filter (UKF).
Consider the objective function $J(\phi,\theta)$ defined in Eq.~\eqref{eq:3.17}.
Then, $J(\phi,\theta)$ converges to zero if and only if the following conditions hold:
\begin{itemize}[leftmargin=*,nosep]
\item (State convergence) The predicted state $\hat{x}_{t+h}$ converges to the groundtruth state $x_{t+h}$. 
\item (Sigma point convergence) The sigma points generated by the UKF converge to the groundtruth trajectory, i.e., $\hat x_{k\mid k-1} = \sum_{i=0}^{2n} W_m^{(i)}\, \mathcal X_{k\mid k-1}^{(i)}\to x_k$. Also, their covariance collapses, i.e., $\sum_{i=0}^{2n}W_c^{(i)}\bigl(\mathcal X_{k\mid k-1}^{(i)}-\hat x_{k\mid k-1}\bigr)
    \bigl(\mathcal X_{k\mid k-1}^{(i)}-\hat x_{k\mid k-1}\bigr)^\top\to0$
\item (Covariance convergence) The predicted covariance matrix $P_{t+h}$ converges to a reference covariance $Q$, where $Q$ satisfies the determinant condition of $\lvert Q\rvert = (2\pi)^{-d}$.
\end{itemize}

\end{theorem}

\begin{proof}
We prove the two directions separately.

\noindent\textbf{Necessity} 
Assume that the forecaster's objective converges to zero, that is, $J(\phi,\theta)\to 0$.
By definition, the objective consists of the prediction error term and the UKF-based likelihood term:
\begin{align}
    J(\phi,\theta)= \big\|F_{\theta}(I_{\phi},\,i_n)-x_{t+h}\big\|_2^{\,2} + L_{\mathrm{UKF}}(\theta,\phi)
\end{align}
Since both terms are nonnegative, each term must converge to zero individually.

First, from $\big\|F_{\theta}(I_{\phi},i_n)-x_{t+h}\big\|_2^{\,2}\to0$, it follows that the predicted state converges to the groundtruth state. Hence, we get the first condition
\begin{align}
    \hat{x}_{t+h}\to x_{t+h}.
\end{align}

Next, consider the UKF loss of
\begin{align}
    L_{\mathrm{UKF}}(\theta,\phi)&=\frac{d}{2}\log(2\pi) +\frac12\log|P_{t+h}| +\frac12 e_{t+h}^\top P_{t+h}^{-1} e_{t+h}\\
    &=\frac{d}{2}\log(2\pi) +\frac12\log|P_{t+h}| +\frac12 (x_{t+h}-\hat{x}_{t+h})^\top P_{t+h}^{-1} (x_{t+h}-\hat{x}_{t+h}),\label{eq:ukf_likeliholod}
\end{align}
where the Mahalanobis term in the right-hand side is always nonnegative. Since $L_{UKF}(\theta,\phi)\to0$ and the Mahalanobis term vanishes as $\hat{x}_{t+h}\to x_{t+h}$, we obtain
\begin{align}
    \frac{d}{2}\log(2\pi) +\frac12\log|P_{t+h}|\to 0, \quad \text{or} \quad |P_{t+h}|\to(2\pi)^{-d}.\label{eq:Pdeterminant_limit}
\end{align}
Under the UKF update, the covariance $P_{t+h}$ is determined by the weighted spread of the sigma points around the predicted mean. Therefore, as the prediction error vanishes and the covariance converges to its limiting value, the sigma points must concentrate around the groundtruth trajectory. In particular, the predicted mean satisfies
\begin{align}
    \hat x_{k\mid k-1} = \sum_{i=0}^{2n} W_m^{(i)}\, \mathcal X_{k\mid k-1}^{(i)}\to x_k,
\end{align}
and the weighted covariance of the sigma points collapses as
\begin{align}
    \sum_{i=0}^{2n}W_c^{(i)}\bigl(\mathcal X_{k\mid k-1}^{(i)}-\hat x_{k\mid k-1}\bigr)
    \bigl(\mathcal X_{k\mid k-1}^{(i)}-\hat x_{k\mid k-1}\bigr)^\top\to 0,
\end{align}
which is our second condition.
Finally, if we denote the limiting covariance matrix as $Q$, then $P_{t+h}\to Q$. Combining with Eq.~\eqref{eq:Pdeterminant_limit}, the determinant condition becomes
\begin{align}
    |Q|=(2\pi)^{-d},
\end{align}
which is our third condition.
This proves the necessity of the stated conditions.

\noindent\textbf{Sufficiency}
Conversely, assume that the sigma points converge to the groundtruth trajectory and that the predicted covariance converges to a reference covariance $Q$, which satisfies $|Q|=(2\pi)^{-d}$.
Because the sigma points converge to the groundtruth state, the predicted mean of UKF also converges as
\begin{align}
    \hat{x}_{k|k-1}\to x_k, \hat{x}_{t+h}\to x_{t+h}.
\end{align}
Therefore, $\big\|F_{\theta}(I_{\phi},\,i_n)-x_{t+h}\big\|_2^{\,2}\to 0$.
It remains to show that the UKF likelihood term of Eq.~\eqref{eq:ukf_likeliholod} also converges to zero.

We first analyze the log-determinant terms in the UKF likelihood.
By Lemma~\ref{lem:lem3.2},
\begin{align}
    \log \det(P_k)-\log \det(Q) = \int_{0}^{1} \operatorname{tr}\!\left[(Q+s(P_k-Q))^{-1}(P_k-Q)\right]\, ds.\label{eq:logPk-logQ}
\end{align}  
By Lemma~\ref{lem:lem3.3}, we can derive the inequality for the inverse term in Eq.~\eqref{eq:logPk-logQ} as:
\begin{align}
    Q+s(P_k-Q) \geq \frac{m}{2}E,
\end{align}
or
\begin{align}
\left\|\bigl(Q+s(P_k-Q)\bigr)^{-1}\right\|_{2}\le \frac{2}{m}.
\label{eq:4.23}
\end{align}
Since $|\operatorname{tr}(B)|\le d||B||_{2}$ holds for an arbitrary matrix $B\in \mathbb{R}^{d\times d}$, the following inequality holds:
\begin{align}
    |\operatorname{tr}\!\left(\bigl(Q+s(P_k-Q)\bigr)^{-1}(P_k-Q)\right)|\le d\left\|\!\bigl(Q+s(P_k-Q)\bigr)^{-1}(P_k-Q)\right\|_{2},\label{eq:trace_ineq}
\end{align}
Then, the following inequality holds by the submultiplicativity of the spectral norm as
\begin{align}
    \left\|\!\bigl(Q+s(P_k-Q)\bigr)^{-1}(P_k-Q)\right\|_{2} \leq || (Q+s(P_k-Q)\bigr)^{-1}||_{2}\cdot||P_k-Q||_{2},
\end{align}
and the inequality of Eq.~\eqref{eq:trace_ineq} is reformulated as below, when combined with the result from Eq.~\eqref{eq:4.23}.
\begin{align}
    |\operatorname{tr}\!\left(\bigl(Q+b(P_k-Q)\bigr)^{-1}(P_k-Q)\right)|
    &\le d\left\|\bigl(Q+s(P_k-Q)\bigr)^{-1}\right\|_2\cdot\|P_k-Q\|_2 \label{eq:trace_upperbound}\\
    &\le d\cdot\frac{2}{m}\|P_k-Q\|_2.\nonumber
\end{align}
Substituting Eq.~\eqref{eq:trace_upperbound} back into Eq.~\eqref{eq:logPk-logQ} yields the following inequalities.
\begin{align}
|\log|P_k|-\log|Q||
&=|\int_{0}^{1}\operatorname{tr}\!\left(\bigl(Q+s(P_k-Q)\bigr)^{-1}(P_k-Q)\right)\,ds | \\
&\le \int_{0}^{1}|\operatorname{tr}\!\left(\bigl(Q+s(P_k-Q)\bigr)^{-1}(P_k-Q)\right)|\,ds\notag\\ 
&\le \int_{0}^{1} d\cdot\frac{2}{m}\,\|P_k-Q\|_{2}\,ds\notag\\
&= d\cdot\frac{2}{m}\,\|P_k-Q\|_{2}.\notag
\end{align}
Therefore, if $P_k \to  Q$, then $\log|P_k|\to\log|Q|$.
Applying this to $P_{t+h}$, we have
\begin{align}
    \frac{1}{2}\log|P_{t+h}|+\frac{d}{2}\log(2\pi)\to
    \frac{1}{2}\log|Q|+\frac{d}{2}\log(2\pi).\label{eq:logp_converge}
\end{align}
Using the condition $|Q|=(2\pi)^{-d}$, the right-hand side of Eq.~\eqref{eq:logp_converge} becomes zero. Hence, the log determinant contribution in the UKF likelihood of Eq.~\eqref{eq:ukf_likeliholod} vanishes.

We next analyze the remaining Mahalanobis term in the UKF likelihood. Since the minimum eigenvalue of $P_{t+h}$ is bounded by Lemma~\ref{lem:lem3.3}, we get $\|P_{t+h}^{-1}\|_2\le \frac{2}{m}$.
Moreover, since $\hat{x}_{t+h}\to x_{t+h}$, the remaining Mahalanobis term in Eq.~\eqref{eq:ukf_likeliholod} is bounded as
\begin{align}
\frac{1}{2}(x_{t+h}-\hat{x}_{t+h})^{\top}P_{t+h}^{-1}(x_{t+h}-\hat{x}_{t+h})
&\le
\frac{1}{2}\|P_{t+h}^{-1}\|_2\,\|x_{t+h}-\hat{x}_{t+h}\|_2^2
\notag\\
&\le
\frac{1}{m}\|x_{t+h}-\hat{x}_{t+h}\|_2^2
\to 0.
\end{align}

Combining these limits yields $L_{UKF}(\theta,\phi)\to0$. Hence, we conclude that $J(\theta,\phi)\to 0$. This proves the sufficiency of the stated conditions.
\end{proof}

Therefore, if we can make our forecaster loss function tend toward zero, we can theoretically satisfy the conditions stated in Theorem~\ref{thm:thm3.4}. Conversely, if the predicted state, sigma points, and covariance satisfy the corresponding convergence condition, then the objective function also converges to zero.

\section{Experiment}\label{sec:sec5}

\subsection{Experimental Settings}
In this section, we empirically evaluate the proposed PDYffusion framework on a diverse set of long-horizon spatiotemporal dynamical prediction tasks. The experiments are designed to assess both distributional accuracy and temporal stability under extended forecasting horizons, where cumulative error and noise amplification are known to be critical challenges. We compare PDYffusion against representative baselines, including 1) DYffusion~\cite{ruhling2023DYffusion}, 2) DDPM~\cite{ho2020denoising}, 3) MCVD~\cite{voleti2022mcvd}, and uncertainty-aware variants based on 4) perturbation~\cite{pathak2022fourcastnet} and 5) dropout~\cite{gal2016dropout}.

We evaluate PDYffusion on four benchmark datasets, which are 1) Sea Surface Temperature (SST), 2) Navier–Stokes, 3) Spring-mesh, and 4) Wave. These datasets cover diverse spatiotemporal systems, including ocean dynamics, incompressible fluid flow, elastic deformation, and oscillatory wave propagation, respectively.
All datasets are evaluated under long-horizon prediction settings, where models must extrapolate system dynamics far beyond the conditioning window. This experimental setup highlights the ability of each method to suppress cumulative error, maintain physical plausibility, and produce well-calibrated predictive distributions over time intervals.

To comprehensively evaluate predictive performance, we consider metrics that capture complementary aspects of long-term forecasting quality. In particular, we report 1) the Continuous Ranked Probability Score (CRPS)~\cite{matheson1976scoring} to assess probabilistic calibration and distributional fidelity, 2) mean squared error (MSE) to measure point-wise accuracy, and 3) the Spread-Skill Ratio (SSR)~\cite{fortin2014should} to quantify temporal smoothness and dynamical consistency of predicted trajectories. Lower values of both CRPS and MSE indicate better predictive accuracy, while SSR values close to 1 indicate better spectral stability and dynamical consistency. All results are averaged over five independent runs to account for stochastic variability.

\subsection{Quantitative Result}\label{sec:sec5.1.1}
Table~\ref{tab:tab2} summarizes the quantitative comparison of PDYffusion with representative baseline models across the benchmark datasets.
On the Navier--Stokes, Spring-mesh, and Wave datasets, PDYffusion consistently achieves the best overall predictive accuracy among all competing methods; with the lowest CRPS of 0.059, 0.0092, and 1.82e-3, and the lowest MSE of 0.017, 4.01e-4, and 1.57e-5 on each dataset, respectively. These results demonstrate the strong probabilistic prediction and calibration accuracies of PDYffusion across diverse dynamical systems.

\begin{table*}[t]
\centering
\small
\setlength{\tabcolsep}{2pt}
\renewcommand{\arraystretch}{1.1}

\begin{subtable}[t]{0.49\textwidth}
\centering
\resizebox{\linewidth}{!}{%
\begin{tabular}{lcccc}
\toprule
Method & CRPS & MSE & SSR & Time [s] \\
\midrule
Perturbation & $0.281 \pm 0.004$ & $0.180 \pm 0.011$ & $0.411 \pm 0.046$ & $0.4241$ \\
Dropout      & $0.266 \pm 0.003$ & \uline{$0.164 \pm 0.004$} & $0.406 \pm 0.042$ & $0.4242$ \\
DDPM         & $0.246 \pm 0.005$ & $0.176 \pm 0.005$ & $0.674 \pm 0.011$ & $0.3054$ \\
MCVD         & \bm{$0.216 \pm 0.056$} & \bm{$0.161 \pm 0.049$} & \uline{$0.926 \pm 0.085$} & $78.9950$ \\
DYffusion    & \uline{$0.224 \pm 0.002$} & $0.173 \pm 0.001$ & \bm{$1.033 \pm 0.005$} & $4.6722$ \\
\midrule
PDYffusion   & $0.226 \pm 0.001$ & $0.166 \pm 0.003$ & $0.683 \pm 0.006$ & $13.0880$ \\
\bottomrule
\end{tabular}}
\caption{SST}
\label{tab:sst_sub}
\end{subtable}
\hfill
\begin{subtable}[t]{0.49\textwidth}
\centering
\resizebox{\linewidth}{!}{%
\begin{tabular}{lcccc}
\toprule
Method & CRPS & MSE & SSR & Time [s] \\
\midrule
Perturbation & $0.090 \pm 0.001$ & $0.028 \pm 0.000$ & $0.448 \pm 0.002$ & 0.0895 \\
Dropout      & $0.078 \pm 0.001$ & $0.027 \pm 0.001$ & $0.715 \pm 0.005$ & 0.0886 \\
DDPM         & $0.180 \pm 0.004$ & $0.105 \pm 0.010$ & $0.573 \pm 0.001$ & 0.0692 \\
MCVD         & $0.152 \pm 0.044$ & $0.070 \pm 0.033$ & $0.524 \pm 0.064$ & 58.4255 \\
DYffusion    & \uline{$0.067 \pm 0.003$} & \uline{$0.022 \pm 0.002$} & \bm{$0.877 \pm 0.006$} & 2.9715 \\
\midrule
PDYffusion   & \bm{$0.059 \pm 0.005$} & \bm{$0.017 \pm 0.002$} & \uline{$0.731 \pm 0.008$} & 10.3917 \\
\bottomrule
\end{tabular}}
\caption{Navier--Stokes}
\label{tab:ns_sub}
\end{subtable}

\vspace{0.6em}

\begin{subtable}[t]{0.49\textwidth}
\centering
\resizebox{\linewidth}{!}{%
\begin{tabular}{lcccc}
\toprule
Method & CRPS & MSE & SSR & Time [s] \\
\midrule
Perturbation & $0.0151 \pm 0.0004$ & $9.05\mathrm{e}{-4} \pm 6.55\mathrm{e}{-5}$ & $1.361 \pm 0.038$ & 0.0152 \\
Dropout      & $0.0138 \pm 0.0006$ & $7.27\mathrm{e}{-4} \pm 6.80\mathrm{e}{-5}$ & \bm{$1.017 \pm 0.029$} & 0.0161 \\
DDPM         & $0.0165 \pm 0.0010$ & $14.13\mathrm{e}{-4} \pm 1.05\mathrm{e}{-4}$ & $0.763 \pm 0.009$ & 0.0098 \\
MCVD         & $0.0146 \pm 0.0094$ & $6.77\mathrm{e}{-4} \pm 1.42\mathrm{e}{-4}$ & $0.782 \pm 0.055$ & 33.4163 \\
DYffusion    & \uline{$0.0103 \pm 0.0022$} & \uline{$4.20\mathrm{e}{-4} \pm 2.10\mathrm{e}{-4}$} & \uline{$1.133 \pm 0.083$} & 1.1005 \\
\midrule
PDYffusion   & \bm{$0.0092 \pm 0.0019$} & \bm{$4.01\mathrm{e}{-4} \pm 2.21\mathrm{e}{-4}$} & $1.204 \pm 0.095$ & 6.2841 \\
\bottomrule
\end{tabular}}
\caption{Spring-mesh}
\label{tab:spring_sub}
\end{subtable}
\hfill
\begin{subtable}[t]{0.49\textwidth}
\centering
\resizebox{\linewidth}{!}{
\begin{tabular}{lcccc}
\toprule
Method & CRPS & MSE & SSR & Time [s] \\
\midrule
Perturbation & $4.88\mathrm{e}{-3}\pm2.30\mathrm{e}{-4}$ & $3.30\mathrm{e}{-5}\pm4.04\mathrm{e}{-6}$ & \uline{$1.164 \pm 0.022$} & 0.0008 \\
Dropout      & $6.51\mathrm{e}{-3}\pm3.18\mathrm{e}{-4}$ & $3.26\mathrm{e}{-5}\pm4.86\mathrm{e}{-6}$ & $1.214 \pm 0.065$ & 0.0008 \\
DDPM         & $9.85\mathrm{e}{-3}\pm2.30\mathrm{e}{-4}$ & $14.31\mathrm{e}{-5}\pm5.27\mathrm{e}{-6}$ & $1.901 \pm 0.079$ & 0.0001 \\
MCVD         & $9.31\mathrm{e}{-3}\pm4.86\mathrm{e}{-4}$ & $26.34\mathrm{e}{-5}\pm6.11\mathrm{e}{-6}$ & $1.400 \pm 0.091$ & 7.2594 \\
DYffusion    & \uline{$2.02\mathrm{e}{-3}\pm3.46\mathrm{e}{-4}$} & \uline{$1.94\mathrm{e}{-5}\pm3.43\mathrm{e}{-6}$} & $0.836 \pm 0.043$ & 0.5258 \\
\midrule
PDYffusion   & \bm{$1.82\mathrm{e}{-3}\pm3.01\mathrm{e}{-4}$} & \bm{$1.57\mathrm{e}{-5}\pm4.50\mathrm{e}{-6}$} & \bm{$0.913 \pm 0.056$} & 3.7763 \\
\bottomrule
\end{tabular}}
\caption{Wave}
\label{tab:wave_sub}
\end{subtable}

\caption{Quantitative results on benchmark datasets}
\label{tab:tab2}
\end{table*}

In contrast, a distinctive behavior can be observed on the SST dataset, where diffusion-based long-horizon forecasting models do not exhibit the same level of performance improvement as observed on the other datasets. One possible explanation is the inherent nature of the SST data. Unlike the synthetic or controlled dynamical datasets, SST represents a real-world large-scale physical system and contains substantial measurement noise and complex environmental variability. In such settings, modeling uncertainty becomes particularly important, and accurately capturing long-term physical dynamics becomes more challenging.
We conjecture that this characteristic of the SST dataset may partly explain why both DYffusion and PDYffusion do not show the same degree of performance improvement as on the other datasets. Interestingly, MCVD achieves the strongest performance on SST in several metrics. One possible interpretation is related to the modeling characteristics of MCVD. As a Markov-chain-based diffusion model, MCVD primarily focuses on short-term temporal transitions by conditioning on past states. In contrast, DYffusion and PDYffusion are designed for long-horizon forecasting using PDE-inspired temporal modeling.
Therefore, it is plausible that the modeling bias of MCVD toward short-term dynamics aligns better with the noisy and locally evolving characteristics of the SST dataset. This observation suggests that the performance differences across models may reflect not only algorithmic improvements but also the interaction between model inductive bias and the underlying properties of the dataset. While this interpretation remains speculative, it highlights an interesting direction for further investigation.

In terms of the SSR metric, although PDYffusion does not always achieve the best performance among the compared methods, its SSR values remain relatively close to the ideal value of 1, indicating stable spectral behavior in long-horizon dynamics. It should be noted that prediction accuracy metrics (i.e., CRPS and MSE) and uncertainty-related metrics (i.e., SSR) often exhibit an inherent trade-off in dynamical forecasting. As models prioritize more accurate reconstruction of the predicted states, deviations in spectral stability may occur. This trade-off is further analyzed in Section~\ref{sec:trade-off}.

\subsection{Effect of Boundary Condition}\label{sec:sec5.1.2}

PDYffusion is trained with a PDE-based regularization term, which enables the model to learn the underlying PDE structure of the dynamical system. An essential component of PDE formulations is the choice of boundary conditions~\cite{evans2022partial}. Since different physical systems are characterized by distinct PDE structures, it is important to examine which boundary conditions allow PDYffusion to perform most effectively. Hence, we consider three commonly used boundary conditions, which are Dirichlet, Neumann, and periodic; and evaluate their impact using CRPS, MSE, and SSR in Table~\ref{tab:tab3}.

\begin{table*}[h]
\centering
\small
\setlength{\tabcolsep}{2pt}
\renewcommand{\arraystretch}{1.0}

\begin{tabular}{lccc ccc ccc ccc}
\toprule
\multirow{2}{*}{Method}
& \multicolumn{3}{c}{SST}
& \multicolumn{3}{c}{Navier--Stokes}
& \multicolumn{3}{c}{Spring-mesh}
& \multicolumn{3}{c}{Wave} \\
\cmidrule(lr){2-4}\cmidrule(lr){5-7}\cmidrule(lr){8-10}\cmidrule(lr){11-13}
& CRPS & MSE & SSR
& CRPS & MSE & SSR
& CRPS & MSE & SSR
& CRPS & MSE & SSR \\
\midrule
Dirichlet
& 0.264 & 0.181 & 0.671
& $\bm{0.059}$ & $\bm{0.017}$  & 0.731
& $\bm{0.0092}$ & $\bm{4.01\mathrm{e}{-4}}$  & $\bm{1.204}$
& $1.94\mathrm{e}{-3}$& $2.33\mathrm{e}{-5}$& 0.853 \\
Neumann
& 0.232 & 0.172 & 0.665
& 0.063 & 0.020 & $\bm{0.735}$
& 0.0105 & $4.16\mathrm{e}{-4}$ & 1.211
& $2.05\mathrm{e}{-3}$ & $2.18\mathrm{e}{-5}$ & 1.110 \\
Periodic
& $\bm{0.226}$ & $\bm{0.166}$ & $\bm{0.683}$
& 0.107 & 0.033 & 0.712
& 0.0121 & $4.39\mathrm{e}{-4}$ & 0.776
& $\bm{1.82\mathrm{e}{-3}}$ & $\bm{1.57\mathrm{e}{-5}}$ & $\bm{0.913}$ \\
\bottomrule
\end{tabular}

\caption{Experiment results on boundary conditions}
\label{tab:tab3}
\end{table*}

On the SST and Wave datasets, periodic boundary conditions achieve the best overall performance, attaining lower CRPS and MSE values while maintaining SSR values closer to 1.  This behavior is consistent with the inherently periodic and spatially continuous nature of ocean temperature fields and wave propagation phenomena, where periodic constraints better preserve spectral continuity.
In contrast, Dirichlet boundary conditions yield the strongest performance on the Navier--Stokes and Spring-mesh datasets. These datasets involve constrained spatial domains and structured boundary interactions, where fixed-value constraints more accurately reflect the physical characteristics of the system.
Neumann boundary conditions generally produce intermediate performance, indicating that gradient-based constraints partially capture physical structure but may not fully enforce spatial consistency.

\subsection{Tuning of Hyperparameter $\lambda$}
\begin{wraptable}[8]{r}{0.35\textwidth}
\vspace{-10pt}
\tiny
\setlength{\tabcolsep}{3pt}
\renewcommand{\arraystretch}{1.15}
\resizebox{\linewidth}{!}{%
\begin{tabular}{cccc}
\toprule
$\lambda$ & CRPS & MSE & SSR \\
\midrule
$\lambda=0.1$ & $0.241$ & $0.178$ & $0.610$ \\
$\lambda=0.2$ & \bm{$0.226$} & \bm{$0.166$} & \bm{$0.683$} \\
$\lambda=0.5$ & $0.391$ & $0.301$ & $0.522$ \\
$\lambda=1.0$ & $0.438$ & $0.326$ & $0.524$ \\
\bottomrule
\end{tabular}}
\caption{Ablation of $\lambda$ on SST}
\label{tab:tab4}
\end{wraptable}

In the proposed PDYffusion, the hyperparameter $\lambda$ is introduced to control the strength of the PDE-based regularization term. As described in Eq.~\eqref{eq:final_loss}, $\lambda$ balances the influence of the PDE constraint and the reconstruction objective in the interpolation loss. When $\lambda$ is too small, the PDE regularization has little effect, and the model behaves similarly to a purely data-driven interpolator. Conversely, excessively large values of $\lambda$ impose overly strong structural constraints that may hinder predictive performance.

To analyze the effect of this hyperparameter, we evaluate PDYffusion under different values of $\lambda$ on the SST dataset.
The results in Table~\ref{tab:tab4} show that a moderate value of $\lambda=0.2$ provides the best overall performance with the lowest CRPS of 0.226 and MSE of 0.166, while also producing an SSR value of 0.683 which is closer to the ideal value of 1.
When $\lambda$ becomes smaller (e.g., $\lambda=0.1$), the influence of PDE regularization becomes weaker, resulting in slightly degraded accuracy and stability.
On the other hand, when $\lambda$ becomes larger (e.g., $\lambda=0.5$ or $1.0$), both CRPS and MSE increase while SSR moves farther away from 1. This suggests that excessively strong PDE constraints can over-regularize the interpolator and negatively affect prediction quality.

\subsection{Accuracy-Stability Trade-off Analysis}\label{sec:trade-off}

In long-horizon dynamical forecasting, predictive accuracy and dynamical stability are often competing objectives. Such competition is also shown in our evaluation metrics. That is, lower MSE indicates higher point-wise prediction accuracy, while SSR values closer to 1 indicate more stable spectral behavior of the predicted dynamics. Therefore, understanding the relationship between these two metrics is important for evaluating whether a model can achieve reliable long-term predictions without sacrificing physical consistency.
To analyze this relationship, we introduce controlled perturbations by gradually injecting Gaussian noise $\mathcal{N}\!\big(0,E)$ into the predicted states. Increased magnitude of Gaussian noise induces higher uncertainty in the prediction process, allowing us to observe how prediction accuracy and spectral stability change under different noise levels.

\begin{wrapfigure}[13]{r}{0.4\linewidth}
\includegraphics[width=\linewidth]{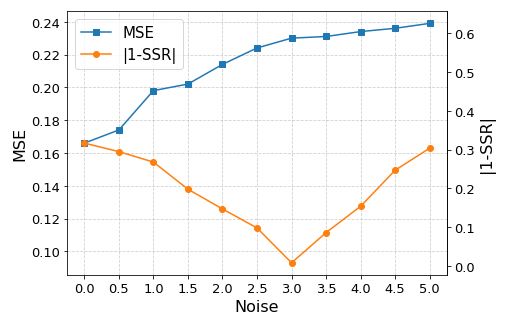}
\caption{Trade-off between accuracy and stability on SST dataset}
\label{fig:trade-sst}
\end{wrapfigure}

Figure~\ref{fig:trade-sst} presents the relationship between MSE and $|1-\mathrm{SSR}|$ for the SST dataset. In the figure, a clear trade-off is observed in the low-noise regime (i.e., noise $\leq 3.0$). In this region, as the magnitude of Gaussian noise increases, the MSE increases, indicating a degradation in prediction accuracy. At the same time, $|1-\mathrm{SSR}|$ decreases instead, meaning that SSR approaches the ideal value of 1 and the spectral stability of the predicted dynamics improves. This behavior demonstrates that increasing uncertainty can help preserve the spectral characteristics of the underlying physical dynamics, even though it reduces point-wise prediction accuracy.
Understanding this trade-off is therefore important when designing forecasting models for long-term dynamical systems.

\subsection{Qualitative Result}
\begin{figure}[!hb]
  \centering
  \includegraphics[width=0.8\textwidth]{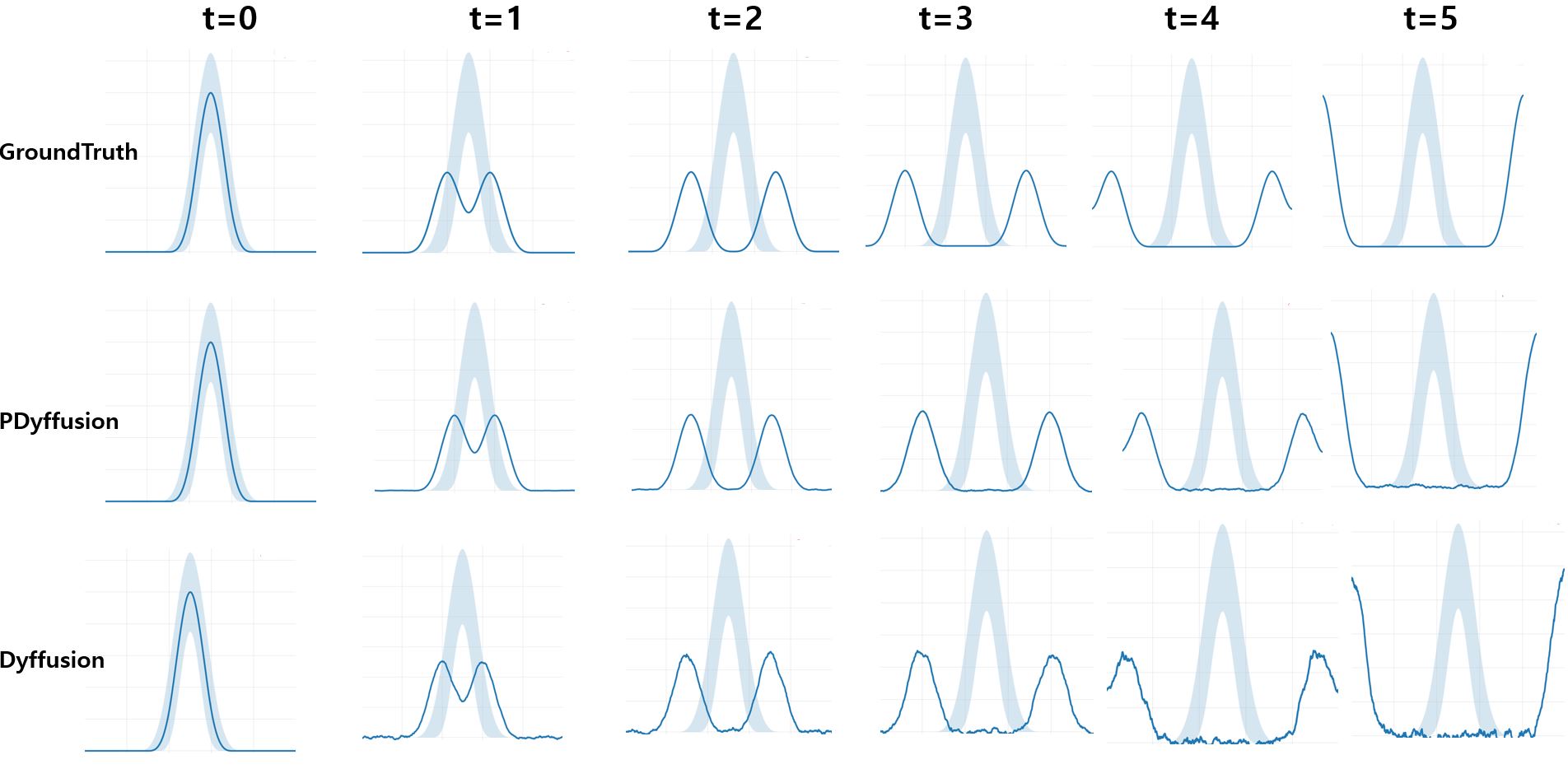}
  \caption{Wave trajectory results of PDYffusion and DYffusion}
  \label{fig:fig3}
\end{figure}

Long-horizon dynamical prediction models often suffer from progressive smoothing or error accumulation, which can eventually lead to unstable trajectories in autoregressive forecasting settings. To qualitatively evaluate the behavior of the proposed model, Figure~\ref{fig:fig3} visualizes representative wave trajectories predicted by DYffusion and PDYffusion together with the groundtruth signals across multiple time steps. In the figure, the solid line represents the predicted trajectory, while the shaded region represents the pulse family from which the initial groundtruth pulse is sampled.

As shown in the figure, both DYffusion and PDYffusion produce reasonable predictions in the early time steps. However, as the prediction horizon $t$ increases, DYffusion gradually exhibits irregular fluctuations and noisy oscillations in the predicted trajectories, particularly in the later time steps. This indicates the accumulation of prediction errors and instability in long-horizon forecasting of DYffusion.
In contrast, PDYffusion preserves the overall structure of the wave trajectory more precisely throughout the entire prediction horizon. The predicted signals remain smooth and maintain the characteristic double-peak structure of the groundtruth dynamics. Moreover, PDYffusion suppresses the spurious oscillatory behavior observed in DYffusion and maintains stable amplitude evolution over time.
This improvement can be attributed to the PDE-based regularization incorporated during training, which encourages the interpolator to respect the underlying physical structure of the dynamical system. By enforcing physically consistent trajectories, the model is able to reduce error accumulation and maintain stable temporal evolution.

Overall, these qualitative observations are consistent with the quantitative results presented in Section~\ref{sec:sec5.1.1} and the theoretical analysis in Section~\ref{sec:sec4}, further supporting the effectiveness of the proposed PDYffusion framework for long-horizon dynamical forecasting.

\section{Conclusion}
\label{sec:sec6}
We propose PDYffusion, a PDE-based diffusion model for improved probabilistic spatiotemporal forecasting. Unlike prior probabilistic forecasting approaches that obtain long-horizon predictions by repeatedly unrolling autoregressive models, PDYffusion is designed to support long-range inference intrinsically by incorporating PDE-derived physical structure as an inductive bias. Moreover, this work provides, to our knowledge, the first comprehensive evaluation of diffusion models across spatiotemporal forecasting tasks. To further mitigate noise accumulation in long-horizon time-series prediction, we incorporate an Unscented Kalman Filter (UKF) update that corrects the predicted trajectory and suppresses propagated stochastic errors. Furthermore, We primarily address the tracking and prediction of spatiotemporal signals governed by PDE-constrained, deterministic dynamics; however, we plan to extend this framework to SPDE-governed regimes in which the underlying evolution is intrinsically stochastic. In particular, we seek to characterize uncertainty transport, long-horizon uncertainty amplification, and the resulting error accumulation in settings where randomness is an inherent component of the system dynamics rather than a byproduct of the observation channel (e.g., financial and physical processes modeled by SDEs/SPDEs). Concretely, taking stochastic systems with explicit driving noise terms, including the Black-Scholes model~\cite{black1973pricing} as representative testbeds, we aim to generalize our PDE-informed gradient-based inductive biases to the SPDE setting. This extension will integrate Bayesian filtering (e.g., UKF-based measurement updates) with stochastic sampling mechanisms to provide principled long-horizon control of noise-induced degradation and uncertainty growth in time-series forecasting.

\clearpage
\appendix

\section{Notation}\label{app:notation}

\begin{table}[H]
\centering
\small
\begin{tabular}{@{}l p{0.85\linewidth}@{}}
\toprule
\textbf{Symbol} & \textbf{Used for} \\
\midrule
$E$ & Identity Matrix \\
$u$ & Groundtruth solution of PDE \\
$\Delta$ & Laplacian \\
$t$ & Index of temporal data snapshots \\
$h$ & Training horizon \\
$N$ & Total number of diffusion steps \\
$n$ & Indexer of the diffusion step, where $0 \le n \le N-1$ \\
$X$ & Distribution from which the data snapshots \\
$\mathcal{X}$ & Sigma point of UKF \\
$W$ & Weight of Sigma point \\
$k$ & Kernel function \\
$w$ & Frequency vector in complex space \\
$x_{t+i}$ & Data point at timestep $t+i$ \\
$\hat{x}_{t+i}$ & Predicted data point at timestep $t+i$ \\
$x_{t:t+h}$ & Shorthand for the sequence of data points $x_t, x_{t+1}, \ldots, x_{t+h}$ \\
$x_t$ & Initial conditions, i.e.\ input snapshot for the forecasting model \\
$x_{t+1:t+h}$ & Sequence of target snapshots that the forecasting model predicts given $x_t$ \\
$I_{\phi}$ & Interpolation neural network \\
$F_{\theta}$ & Forecasting neural network\\
$i_n$ & Mapping of diffusion step $n$ to timestep $i_n$\\
$U[a,b]$ & Uniform distribution \\
\bottomrule

\end{tabular}
\caption{Notation for PDYffusion}
\label{tab:tab1}
\end{table}
\FloatBarrier

\section{Details of Lemma}\label{app:appA}

In the proof of Section~\ref{sec:sec4}, Lemma~\ref{lem:lem4.2}, Lemma~\ref{lem:lem3.2} and Lemma~\ref{lem:lem3.3} were used. We provide the detailed proofs of the lemmas below.

\subsection{Proof of Lemma~\ref{lem:lem4.2}}\label{proof:lem4.2}

According to the definition of $I_{\phi}(x_t,x_{t+1},i) =u(x_t,x_{t+1},i)+\epsilon_{I_{\phi}}$, it satisfies that:
\begin{align}
    \epsilon_{I_{\phi}} = I_{\phi}-u = T_{\lambda}I_{base}-u = T_{\lambda}(u+\epsilon_{I_{base}})-u = u(T_{\lambda}-E)+T_{\lambda}\epsilon_{I_{base}}
    \label{eq:A1}
\end{align}
Furthermore, by assumption, $Au = 0$ implies $u\in ker(A)$.
Since $u$ consists only of the component corresponding to the eigenvalue 0 of $A$, the following holds:
\begin{align}
    T_{\lambda}u = (E+\lambda A)^{-1}u = u.
    \label{eq:A2}
\end{align}
By Eq.~\eqref{eq:A2}, the following equality holds.
\begin{align}
    T_{\lambda}u-u  = (T_{\lambda}-E)u = 0
    \label{eq:A3}
\end{align}
If the result of Eq.~\eqref{eq:A3} is applied to Eq.~\eqref{eq:A1}, it satisfies the following result.
\begin{align}
    \epsilon_{I_{\phi}} = T_{\lambda}\epsilon_{I_{base}}
\end{align}

Meanwhile, since $\epsilon_{I_{base}} = \Sigma_{j}\xi_{j}e_{j}$ and $Ae_j = \mu_{j}$ hold, we can write the following:
\begin{align}
    T_{\lambda}e_{j} = (E+\lambda A)^{-1}e_{j} = \frac{1}{1+\lambda \mu_{j}}e_j.
\end{align}
Therefore, it satisfies as below.
\begin{align}
    \epsilon_{\phi} = T_{\lambda}\epsilon_{base} = \Sigma_{j} T_{\lambda} \xi_{j} e_{j} = \Sigma_{j}\frac{1}{1+\lambda \mu_{j}}\xi_{j}e_j
\end{align}
Furthermore, since $\mu_{j} \ge 0$ and $\lambda \ge 0$, it follows that $0\le\frac{1}{1+\lambda \mu_{j}}\le1$ . This suggests that PDE regularization acts as a filter that suppresses each mode.

Meanwhile, by definition, $\Sigma_{base}$ satisfies the following
\begin{align}
    \Sigma_{base} = cov(\epsilon_{base}) = \Sigma_{j} var(\xi_{j})e_{j}e_{j}^{T}
\end{align}
Similarly, since $\epsilon_{\phi} = \Sigma_{j}\frac{1}{1+\lambda \mu_{j}}var(\xi_{j})e_j$, the following holds.
\begin{align}
    \Sigma_{\phi} = cov(\epsilon_{\phi}) = \Sigma_{j}\frac{1}{(1+\lambda \mu_{j})^{2}}\xi_{j}e_je_j^{T}
\end{align}
If we choose an arbitrary $w = \Sigma_{j}w_{j}e_{j}\in \mathcal{H}$, the following inequality holds by the equivalent characterization of the Loewner order~\cite{bartz2017resolvent}.
\begin{align}
    <w,\Sigma_{\phi}w> = \Sigma_{j} \frac{1}{(1+\lambda\mu_{j})^{2}}var(\xi)w_{j}^{2}\le \Sigma_{j}var(\xi)w_{j}^{2} = <w,\Sigma_{base}w>
\end{align}
Since this inequality holds for all $w\in \mathcal{H}$, it is equivalent to the following.
\begin{align}
    \Sigma_{\phi}\le\Sigma_{base}
\end{align}

\subsection{Proof of Lemma~\ref{lem:lem3.2}}\label{proof:lem3.2}

Let matrices $A$ and $B$ be invertible matrices, $t\in[0,1]$, $X(t) = B + t(A-B), X(0)=B, X(1)=A$, and $f(t) = \log\det X(t)$. Then, it holds that $f(1) - f(0) = \log\det A - \log\det B$.

Moreover, taking derivative of $f(t)$ yields 
\begin{align}
    f'(t)=\frac{d}{dt}\log\det X(t)= \frac{1}{\det X(t)}\frac{d}{dt}\det X(t) =\mathrm{tr}\!\left(X(t)^{-1}X'(t)\right).
\end{align}
Also, since $X'(t) = A-B$ holds, then it also holds that 
\begin{align}
    \log\det A - \log\det B=\int_0^1\mathrm{tr}\!\left((B+t(A-B))^{-1}(A-B)\right)\,dt.
\end{align}

\subsection{Proof of Lemma~\ref{lem:lem3.3}}\label{proof:lem3.3}

We assume that $mE\preceq P_k\preceq ME$ and $P_k \to Q$, where $P_k$ and $Q$ are symmetric and $M\le m < 0$.
Then, it holds that 
\begin{align}
 mE\preceq Q\preceq ME.
\end{align}
Therefore, for every eigenvalue $\lambda$, the followings hold:
\begin{align}
    \lambda_{\min}(Q) \ge m,
\quad
\lambda_{\max}(Q) \le M.
\label{eq:a.4}
\end{align}
By combining Eq.~\eqref{eq:a.4} with Weyl inequality~\cite{weyl1912asymptotische}, which states $|\lambda_i(P_k)-\lambda_i(Q)| \le \|P_k-Q\|_2, (i=1,\dots,n)$, we get
\begin{align}
\lambda_{\min}(P_k)
&\ge \lambda_{\min}(Q)-\|P_k-Q\|_2
\ge m-\|P_k-Q\|_2, \label{eq:A5}\\
\lambda_{\max}(P_k)
&\le \lambda_{\max}(Q)+\|P_k-Q\|_2
\le M+\|P_k-Q\|_2. \label{eq:A6}
\end{align}

Since $P_k \to Q$, we have $\|P_k-Q\|_2 \to 0$. Also, since $m=\lambda_{\min}(Q)>0$, we may apply the definition of convergence with $\varepsilon = \frac{m}{2} >0$. Hence, there exists $K\in\mathbb{N}$ such that $\forall$ $k\ge K$,
\begin{align}
\|P_k-Q\|_2 < \frac{m}{2}.
\end{align}
Using the bounds in Eq.~\eqref{eq:A5} and \eqref{eq:A6}, for every $k\ge K$ we obtain
\begin{align}
\lambda_{\min}(P_k) \ge m-\frac{m}{2}=\frac{m}{2},
\end{align}
and
\begin{align}
\lambda_{\max}(P_k)
\le M+\frac{m}{2}
\le M+\frac{M}{2}
=\frac{3M}{2}
\le 2M,
\end{align}
where we used $m\le M$.
Therefore, for all sufficiently large $k$,
\begin{align}
\operatorname{spec}(P_k)\subset \left[\frac{m}{2},\,2M\right].
\end{align}

On the other hand, since $mE\preceq Q\preceq ME$, we already have
\begin{align}
\operatorname{spec}(Q)\subset [m,M]\subset \left[\frac{m}{2},\,2M\right].
\end{align}
Consequently, for all sufficiently large $k$,
\begin{align}
\operatorname{spec}(P_k),\ \operatorname{spec}(Q)
\subset \left[\frac{m}{2},\,2M\right].
\end{align}

\FloatBarrier

\section{Code Implementation}

We provide our implementation codes as the GitHub link of \url{https://github.com/minyoung445/Dynamic-informed-Diffusion-model-Through-PDE-based-sampling-and-Filtering-method}

\end{document}